\documentclass[conference]{IEEEtran}
\ifCLASSINFOpdf
   \usepackage[pdftex]{graphicx}
\else
   \usepackage[dvips]{graphicx}
\fi

\usepackage[cmex10]{amsmath}
\usepackage{amssymb}
\usepackage{amsthm}
\usepackage{mathtools}
\usepackage[hidelinks]{hyperref}
\usepackage{cite}
\usepackage{soul}
\setstcolor{red}
\usepackage{lipsum}
\usepackage{xargs}
\usepackage{booktabs}
\usepackage{multirow}
\usepackage{algorithm}
\usepackage{algpseudocode}
\usepackage{tabularx}
\usepackage[normalem]{ulem}
\useunder{\uline}{\ul}{}
\usepackage[table,xcdraw]{xcolor}
\usepackage{cite}
\usepackage[caption=false,font=footnotesize]{subfig}

\newtheorem{theorem}{Theorem}
\newtheorem{definition}{Definition}
\newtheorem{corollary}{Corollary}[theorem]

\makeatletter
\let\old@ps@headings\ps@headings
\let\old@ps@IEEEtitlepagestyle\ps@IEEEtitlepagestyle
\def\psccfooter#1{%
    \def\ps@headings{%
        \old@ps@headings%
        \def\@oddfoot{\strut\hfill#1\hfill\strut}%
        \def\@evenfoot{\strut\hfill#1\hfill\strut}%
    }%
    \def\ps@IEEEtitlepagestyle{%
        \old@ps@IEEEtitlepagestyle%
        \def\@oddfoot{\strut\hfill#1\hfill\strut}%
        \def\@evenfoot{\strut\hfill#1\hfill\strut}%
    }%
    \ps@headings%
}
\makeatother

\def\thanksto#1{
\begingroup
\def\thefootnote{}
\footnote{
\kern -3pt
\hrule width 0.4\columnwidth height 0.2pt
\kern 5pt
#1
}
\setcounter{footnote}{0}
\endgroup
}
\begin{document}

\title{Stable Training of Probabilistic Models Using the Leave-One-Out Maximum Log-Likelihood Objective}

\author{
\IEEEauthorblockN{Kutay Bölat, Simon H. Tindemans, Peter Palensky}
\IEEEauthorblockA{Department of Electrical Sustainable Energy \\
Technische Universiteit Delft\\
Delft, The Netherlands\\
\{K.Bolat, S.H.Tindemans, P.Palensky\}@tudelft.nl}
}

\maketitle
\global\csname @topnum\endcsname 0
\global\csname @botnum\endcsname 0
\begin{abstract}
    Probabilistic modelling of power systems operation and planning processes depends on data-driven methods, which require sufficiently large datasets. When historical data lacks this, it is desired to model the underlying data generation mechanism as a probability distribution to assess the data quality and generate more data, if needed. Kernel density estimation (KDE) based models are popular choices for this task, but they fail to adapt to data regions with varying densities. In this paper, an \textit{adaptive KDE} model is employed to circumvent this, where each kernel in the model has an individual bandwidth. The \textit{leave-one-out maximum log-likelihood} (LOO-MLL) criterion is proposed to prevent the singular solutions that the regular MLL criterion gives rise to, and it is proven that LOO-MLL prevents these. Relying on this guaranteed robustness, the model is extended by adjustable weights for the kernels. In addition, a \textit{modified expectation-maximization} algorithm is employed to accelerate the optimization speed reliably. The performance of the proposed method and models are exhibited on two power systems datasets using different statistical tests and by comparison with Gaussian mixture models. Results show that the proposed models have promising performance, in addition to their singularity prevention guarantees.
\end{abstract}

\thanksto{\noindent This project has received funding from the European Union's Horizon 2020 research and innovation programme under the Marie Skłodowska-Curie grant agreement No 956433.}

\section{Introduction}
\begin{figure}
    \centering
    \includegraphics[width=\linewidth]{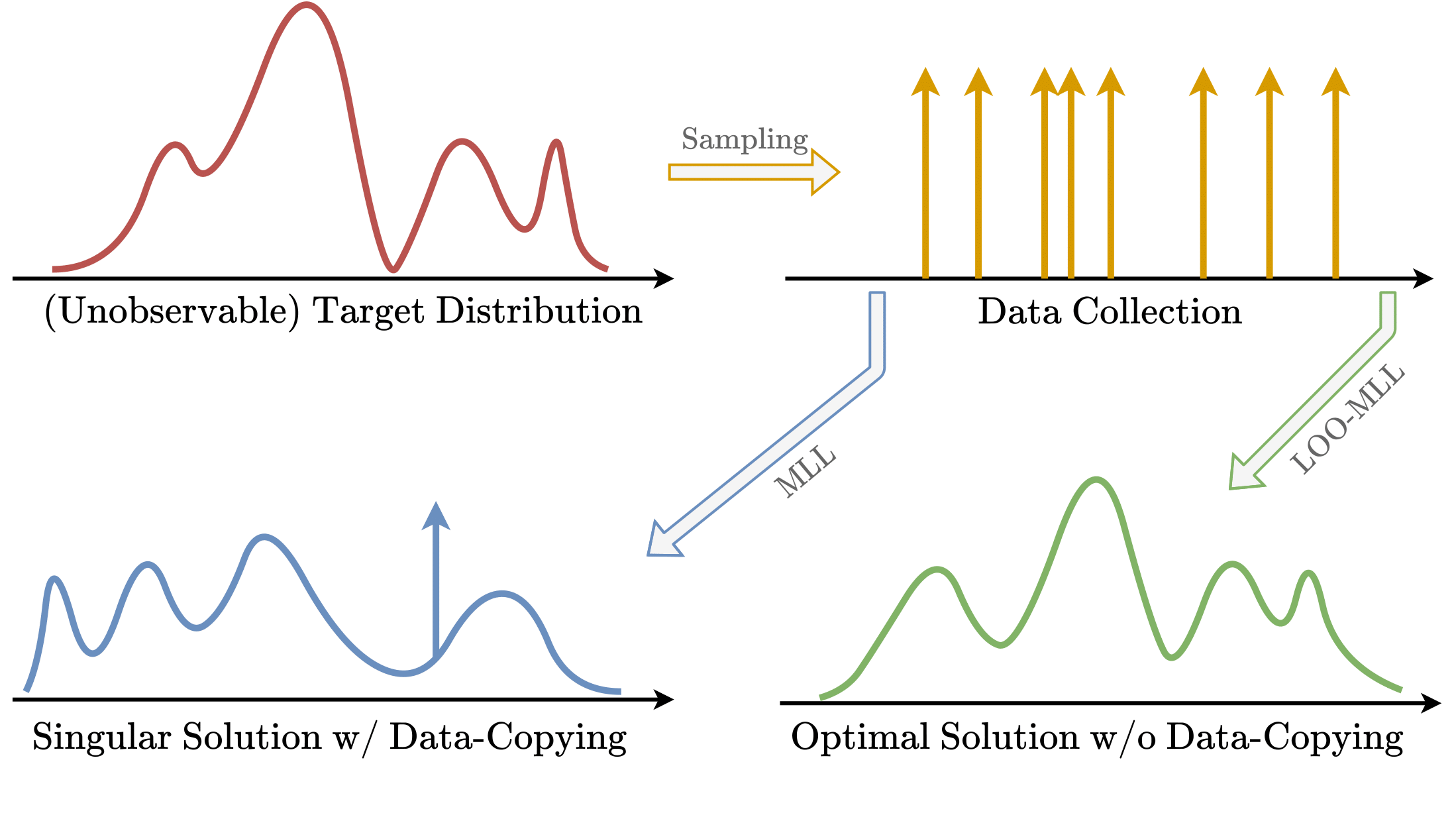}
    \vspace{-0.5cm}
    \caption{Overall objective of probabilistic modelling. Objective functions (MLL, LOO-MLL) can only use the data points. MLL directly targets the dataset, and it copies a random data point (singularity). LOO-MLL avoids this, and the resulting model is similar to the real target distribution.}
    \label{fig:summary}
\end{figure}

Today's power systems exhibit unprecedented levels of variability, especially due to the high penetration of renewable energy systems. Understanding these variabilities is crucial for the operation and planning of power systems because good models that represent these uncertainties can aid the decision-making process of system operators. Data-driven methods are the go-to approaches for such modelling tasks, but the effectiveness of these methods depends on the data quality, including its abundance, representativeness and health (missing values, outliers, etc.). Thus, assessing and improving the quality of real-life data is crucial for effective modelling.

\begin{figure}[h!]
    \centering
    \includegraphics[width=.95\linewidth]{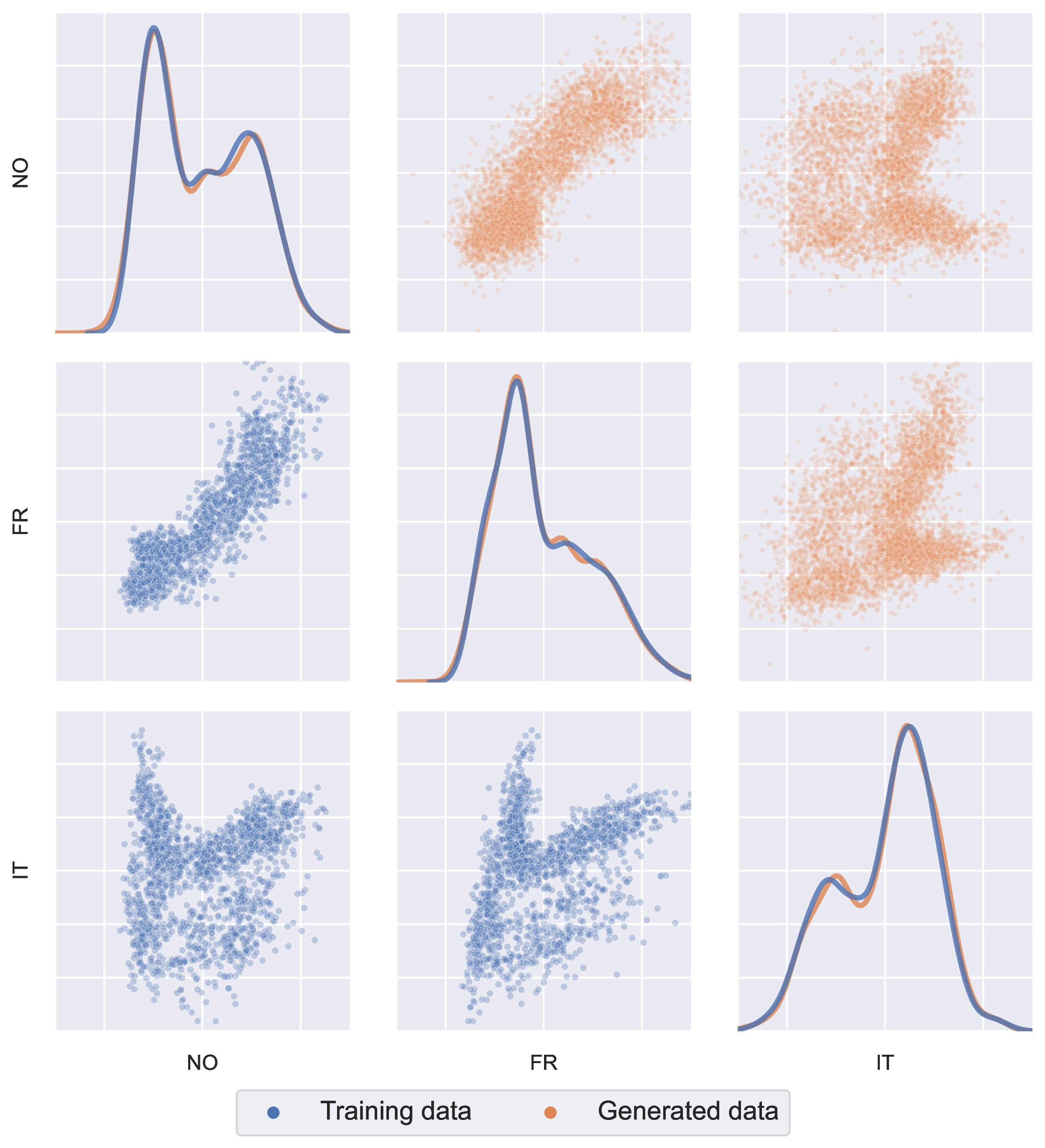}
    \caption{Visual comparison between the training and the generated daily consumption data for three countries (NO, FR, IT) from the Europe dataset (Section \ref{subs:datasets}), illustrating complex dependencies.}
    \label{fig:samples}
\end{figure}

One way to achieve these goals is the \textit{data-driven probabilistic modelling} of the data, which aims to find the closest distribution to the unobservable real-life distribution that generated the data as a probability density function (\textit{pdf}). Fig. \ref{fig:summary} illustrates this process and Fig. \ref{fig:samples} indicates both the complexity of dependencies and the ability to use such a model to generate more data. These can aid various data-driven applications such as security analysis \cite{konstantelos2018using}, and anomaly detection \cite{wang2023improved}.

There is a wide variety of methods for the density estimation problem such as copulas\cite{konstantelos2018using}, Gaussian mixture models (GMMs)\cite{murphy2012machine}, and variational autoencoders\cite{kingma2013auto}. One of the most common among these is kernel density estimation (KDE)\cite{murphy2012machine}, which relies on placing kernels centred on data points and averaging them to form a pdf. However, the regular KDE model has an identical bandwidth parameter for all its kernels. This could lead to (1) noisy samples from the probabilistic model and (2) suboptimal estimation of the pdf in low-density regions where the data is scarce. This challenge can be overcome by assigning individual bandwidths for each data point with respect to their relative locations among each other \cite{sain1994adaptive, van2017variable}. Thus, this \textit{adaptive} KDE (A-KDE) model requires an optimization objective for its bandwidths to be assigned.

The maximum log-likelihood (MLL) criterion is one of the most well-established objectives for optimizing probabilistic models. Yet, the high flexibility of the adaptive KDE model leads to a phenomenon called \textit{data-copying}\cite{meehan2020non}. This occurs when the A-KDE model is optimized to the extent that at least one bandwidth converges to zero and re-produces the data point as seen in Fig. \ref{fig:summary} (lower left) and results in a \textit{singular} solution that causes the optimization algorithm to stop while the rest of distribution is arbitrarily shaped. Please note that most, if not all, of the highly flexible probabilistic models, such as GMMs and variational autoencoders, are prone to this phenomenon\cite{meehan2020non}. 

Avoiding data-copying can be crucial for applications where the probabilistic modelling is part of a pipeline. For instance, a reliability assessment pipeline can have its probabilistic asset modelling task automated for a smoother decision-making process, or edge computing devices in substations can utilize probabilistic models to automatically update their detection algorithms. These applications necessitate guarantees of non-singular solutions for providing reliable data processing pipelines, especially if the resources (computation power, time, battery, etc.) are limited.

\textbf{Contributions.} In this work, we explore a singularity mitigation strategy called the leave-one-out (LOO) MLL criterion using KDE-based models \cite{van2017variable} in a similar fashion to the jackknife estimation method in statistics \cite{mcintosh2016jackknife}\footnote{The jackknife estimation method systematically resamples the dataset by leaving out one observation at a time to reveal the effect of individual data points on the estimator. It can be used for cross-validation and bias-variance estimation.}. The contributions are as follows.
\begin{itemize}
    \item We prove that KDE-based models converge to singular solutions when the regular MLL criterion is employed.
    \item We introduce the LOO-MLL objective to the KDE-based models and prove that it results in well-behaved and robust solutions.
    \item We propose $\pi$-KDE as a more flexible extension of A-KDE by integrating kernel weights into the model.
    \item We propose a modified expectation-maximization (EM) optimization procedure to exploit the advantages of EM.
    \item We developed a testing procedure to compare the performance of probabilistic models quantitatively. 
\end{itemize}

\section{Problem Definition}

\subsection{Motivating Examples}
We start by giving some motivating examples of probabilistic modelling for power system applications.

\subsubsection{Synthetic data generation} 
Smart meter data holds great potential for power systems operation and planning. Unfortunately, accessing these data is not generally possible due to privacy concerns. Moreover, even if the data is accessible, it might not be abundant enough due to its historic nature. These challenges can be overcome by data-driven modelling of the smart meter data at hand so that new datasets following a similar distribution can be generated by taking -effectively unlimited- samples from the probabilistic model. However, the data-copying phenomenon can hinder the privacy-preservation and expressiveness of the model due to the generation of exact replicas of certain data points in the original dataset.

\subsubsection{Rare-event sampling}
Certain events, such as extreme weather conditions, are crucial for various power systems applications like reliability assessment and predictive maintenance. Data points belonging to these events tend to be scarce, which may undermine the results of the application by creating a bias towards a few specific usual events. Probabilistic modelling can remedy this if the distribution's tail is modelled properly so one can take samples from the tail and enrich the rare-event data with unseen rare events. Yet, copying the rare-event data in the dataset interferes with the generalization of the model in the tail regions and inhibits the enrichment.

\subsection{Preliminaries}

\subsubsection{Maximum Log-likelihood Criterion}

Let $p(\mathbf{x};\theta):\mathbb{R}^d\rightarrow\mathbb{R}_{\geq0}$ represent the (multi-dimensional) pdf that we employ as our parametric distribution model where $\theta = \{\theta_a,\theta_f\}$ is the union of adjustable ($\theta_a$) and fixed ($\theta_f$) model parameters. Using the dataset $\mathcal{X}=\{\mathbf{x}_i\}_{i=1}^N$, we aim to find the best model that captures the underlying data generation process. The most common approach to represent this aim as a mathematical objective is using the MLL criterion,
\begin{equation}\label{eq:mll}
    \theta_a^* = \underset{\theta_a}{\text{argmax}} \frac{1}{N}\sum_i \log p(\mathbf{x}_i;\theta),
\end{equation}
where we hereon call the $\frac{1}{N} \sum_i \log p(\mathbf{x}_i;\theta)$ term as the \textit{total log-likelihood}. Intuitively, this criteria incentivizes the model to cover as many data points as possible in its high-density regions since likelihood is an indicator of expected frequency.

\subsubsection{Kernel Density Estimation}
KDE is a methodology that utilizes the data points in the dataset to parameterize the probabilistic model. This is accomplished by placing \textit{identical} kernels centred on data points and averaging them. For conciseness, we focus on Gaussian kernels in this work. Thus, the pdf of the KDE-based model can be written as
\begin{equation}
    p_\text{KDE}(\mathbf{x};\theta) = \frac{1}{N}\sum_j\mathcal{N}(\mathbf{x};\mu_j=\mathbf{x}_j,\Sigma_j=\sigma^2\mathbf{I})
\end{equation}
which essentially is a mixture distribution model with uniform weights. Conventionally, the parameters of the KDE model are fixed, i.e. $\theta_a=\varnothing$, and the \textit{bandwidth} parameter, $\sigma>0$, is the same for all kernels. As a result, the applicability of KDE-based models to high-dimensional datasets is limited due to decreasing locality with increasing dimensionality \cite{sain1994adaptive,van2017variable}.

\subsection{Adaptive Kernel Density Estimation}

The aforementioned drawbacks of KDE models motivate us to employ \textit{individual} bandwidths for each kernel to have the flexibility of adapting the coverage of the model locally. This adaptive KDE (A-KDE) model is defined as
\begin{equation}
    p_\text{A-KDE}(\mathbf{x};\theta) = \frac{1}{N}\sum_j\mathcal{N}(\mathbf{x};\mu_j=\mathbf{x}_j,\Sigma_j=\sigma_j^2\mathbf{I}).
\end{equation}
where $\sigma_j>0, ~ \forall j$. Additionally, we let the bandwidths be adjustable, i.e. $\theta_a=\{\sigma_j\}^N_{j=1}$. 

\subsection{Data-copying as a Singularity}
The additional flexibility and adjustability that come with A-KDE encourage us to optimize $\theta_a$ according to the MLL criterion in \eqref{eq:mll}. Unfortunately, the direct employment of this criterion as an objective function results in one or more bandwidth parameters converging to zero (\textit{bandwidth-collapse}). Since a kernel with zero bandwidth contains no uncertainty and precisely represents the data point, we call this phenomenon \textit{data-copying}.

\begin{definition}[Data-copying] 
    An A-KDE model copies a data point $\mathbf{x}_{j'}\in\mathcal{X}$ when $\sigma_{j'} \rightarrow 0^+$. Thus, the data-copying phenomenon occurs when $ \exists j: \sigma_j \rightarrow 0^+$.
\end{definition}

\begin{theorem}\label{thm:data_copy}
    An A-KDE model optimized by MLL objective copies at least one data point if and only if the total log-likelihood goes to infinity, i.e.
    \begin{equation*}
        \exists j: \sigma_j \rightarrow 0^+ \Longleftrightarrow \sum_i \log p_\text{A-KDE}(\mathbf{x}_i;\theta) \rightarrow \infty.
    \end{equation*}
\end{theorem}
The proof of the Theorem \ref{thm:data_copy} can be found in App. \ref{app:data_copy_proof}. This theorem implies that data-copying is a property of the global optimizer of the MLL objective, resulting in a \textit{singular} solution for the density estimation problem. 

Intuitively, the MLL objective drives the model to replicate the empirical data distribution by copying all the data points, i.e. $\sigma_j\rightarrow0^+,~\forall j$. Ideally, these bandwidth-collapses continue throughout the optimization until the full replication of the dataset. However, as we have shown, one bandwidth-collapse is enough to take the total log-likelihood to infinity. In practice, numerical optimization algorithms cannot handle infinite values and stop the optimization before the other bandwidths collapse. Thus, besides being singular, we can define the data-copying phenomenon also as an \textit{unstable} solution.

\section{Methodology}
The aforementioned challenges that come with the increased flexibility of A-KDE models require a mitigation mechanism for healthy optimization. Since we want our method to be applicable to the problem without any prior assumptions regarding data, we rule out ad-hoc methods limiting the flexibility of the model, like regularization.

\subsection{Leave-One-Out Maximum Log-Likelihood Objective for Adaptive Kernel Density Estimation} \label{subs:loo_a-kde}
In order to mitigate the bandwidth-collapses, we should look at the root of the problem. Intuitively, it is more rewarding for the kernels to assign higher likelihoods to the data points that they centred on (\textit{self-contribution}), which drives these kernels to ignore the surrounding data points. This can also be seen in the data-copying proof in App. \ref{app:data_copy_proof}. 

A natural solution to this problem would be to modify the MLL objective in a way that we leave these self-contributions out of the total log-likelihood, i.e.
\begin{equation}\label{eq:LOO-MLL_objective}
    \{\sigma_j^*\}_j = \underset{\{\sigma_j\}_j}{\text{argmax}} \sum_i \log \sum_{j\neq i} \mathcal{N}(\mathbf{x}_i;\mathbf{x}_j,\sigma^2_j\mathbf{I}).
\end{equation}
We call this the LOO-MLL objective for A-KDE. We guarantee this objective solves the data-copying problem by the following theorem, and its proof can be found in App. \ref{app:loo_wellbehave}.
\begin{theorem}\label{thm:loo_wellbehave}
    Data-copying cannot occur for any optimal solution for the modelling problem with A-KDE if the LOO-MLL objective is used and there are no repeating data points in the dataset.
\end{theorem}
Note that the unique data points assumption holds almost surely (with probability one) for non-discrete data domains, so that the data-copying is not a problem for datasets drawn from continuous distributions. In addition, we also show that the instability problem that we encounter in the regular MLL objective does not occur when we employ the LOO-MLL objective.
\begin{theorem}\label{thm:loo_stability}
    The total log-likelihood in \eqref{eq:LOO-MLL_objective} is always bounded from above if there are no repeating data points in the dataset.
\end{theorem}
\begin{proof}
    Let us define $m \vcentcolon=\min_{\{i,j:i\neq j\}}(\lVert \mathbf{x}_i - \mathbf{x}_j \rVert)$. Thanks to the no-replica data point assumption we have $m>0$, and we can derive
    \begin{equation}
    \begin{split}
        &\mathcal{N}(\mathbf{x}_i;\mathbf{x}_{j\neq i},\sigma^2_j\mathbf{I}) \propto  \sigma^{-d}_j\exp(-\frac{\lVert \mathbf{x}_i - \mathbf{x}_{j\neq i} \rVert^2}{2\sigma^2_j}) \\
        & \leq \frac{1}{\sigma^d_j}\exp(\frac{-m^2}{2\sigma^2_j}) \leq \frac{d^{\frac{d}{2}}}{\exp(\frac{d}{2})}m^{-d} = c < \infty
    \end{split}
    \end{equation}
    As a result, 
    \begin{equation}
        \sum_i \log \sum_{j\neq i} \mathcal{N}(\mathbf{x}_i;\mathbf{x}_j,\sigma^2_j\mathbf{I}) < N \log ((N-1)c) < \infty
    \end{equation}
    which concludes the proof.
\end{proof}
Consequently, employing LOO-MLL objective to non-discrete datasets almost surely guarantees the prevention of A-KDE's drawbacks, namely data-copying and instability.

\subsection{$\pi$-Kernel Density Estimation}\label{subs:loo_pi-kde}
As we mentioned before, the KDE-based models are essentially mixture models, and A-KDE models are no exceptions. We can use this resemblance to extend their flexibility by introducing individual weights to each kernel as
\begin{equation}
    p_{\pi\text{-KDE}}(\mathbf{x};\theta)=\sum_j \pi_j \mathcal{N}(\mathbf{x};\mathbf{x}_j,\sigma^2_j\mathbf{I})
\end{equation}
where $\pi_j>0$, $\sum_j \pi_j =1$ and $\theta_a=\{\pi_j,\sigma_j\}_j$. Note that A-KDE is a special form of $\pi$-KDE with $\pi_j=\frac{1}{N},\forall j$. Accordingly, the LOO-MLL objective for $\pi$-KDE models can be written in a similar manner too:
\begin{equation}
    \{\pi^*_j,\sigma_j^*\}_j = \underset{\{\pi_j,\sigma_j\}_j}{\text{argmax}} \sum_i \log \sum_{j\neq i} \pi_j \mathcal{N}(\mathbf{x}_i;\mathbf{x}_j,\sigma^2_j\mathbf{I}).
\end{equation}
This integration of the kernel weights introduces greater flexibility thanks to the higher number of parameters. Additionally, we hypothesize that the model's sensitivity to outliers is reduced by employing this method since, intuitively, the kernels belonging to these outlier data points tend to have smaller weights. The analysis of this claim is outside of the scope of this study, and we leave it for future work.

\begin{corollary}
    The theorems \ref{thm:data_copy}-\ref{thm:loo_stability} apply to the $\pi$-KDE model and its related LOO-MLL objective.
\end{corollary}
This corollary can be proved directly by including the weights to the derivations in the corresponding proofs. Intuitively, the $\pi$-KDE model employs a convex combination of the likelihoods in the A-KDE model and adapting this convexity to the given proofs does not change the results. 

\subsection{Modified Expectation-Maximization Algorithm}
 Until now, no specific optimization algorithm has been indicated to find the optimal solutions for the aforementioned LOO-MLL problems in \ref{subs:loo_a-kde} and \ref{subs:loo_pi-kde}. Because of their continuous nature, a wide variety of off-the-shelf automatic differentiation-based optimizers, such as Adam\cite{kingma2014adam} are applicable to our problem. These optimizers provide a seamless first-order gradient-based optimization for a given model and objective function.
 
 On the other hand, the A-KDE/$\pi$-KDE models are special cases of isotropic Gaussian mixture models where the centres are predetermined and fixed, suggesting that we can employ the EM algorithm\cite{murphy2012machine} for its desirable properties and intuitive implementation. Thus, the expectation and maximization steps of the algorithm can easily be applied to the conventional MLL objective. However, we must modify this EM algorithm according to our LOO-MLL objective to obtain well-behaved and stable solutions. 
 
 Accordingly, we  propose to use the following modified EM algorithm to iteratively maximize the LOO-MLL objective:
 
\begin{itemize}
    \item \textbf{E-step:}
        \begin{equation}
            r_{ij} = \left\{ 
            \begin{array}{cl}
             \frac{\pi_j\mathcal{N}(\mathbf{x}_i;\mathbf{x}_j,\sigma^2_j\mathbf{I})}{\sum_{j' \neq i} \pi_{j'} \mathcal{N}(\mathbf{x}_i;\mathbf{x}_{j'},\sigma^2_{j'}\mathbf{I})}& : \ i \neq j \\
            0 & : \ i=j
            \end{array} \right.
        \end{equation}
    \item \textbf{M-step:}
        \begin{equation}
            \sigma^2_j = \frac{1}{d}\frac{\sum_i r_{ij} \lVert \mathbf{x}_i -\mathbf{x}_j \rVert^2 }{\sum_i r_{ij}}
        \end{equation}
        \begin{equation}
            \pi_j = \frac{1}{N}\sum_i r_{ij}
        \end{equation}
\end{itemize}
The weights are fixed to $\pi_j=\frac{1}{N},\forall j$ for A-KDE. Note that the M-step remains the same as the M-step of the regular EM algorithm, thanks to the assignment of zero responsibilities to the self-kernels, i.e. $r_{ii}=0$. This assignment is a representation of the LOO mechanism in a way that the data points have no effect on the optimization of their self-kernels.

\section{Experimentation}
In order to exhibit the probabilistic modelling capabilities of A-KDE and $\pi$-KDE models, we run a number of experiments. This section describes the datasets, the model comparison strategy and the experiment settings.\footnote{\url{https://github.com/kabolat/leave-one-out_maximum-log-likelihood}.}

\subsection{Datasets} \label{subs:datasets}
We employ the data from ENTSO-E Transparency Platform\footnote{\url{https://transparency.entsoe.eu/}}\footnote{\url{https://data.open-power-system-data.org/time_series/}} as the basis of our datasets. We curated two datasets using this platform\footnote{The datasets can be found in the shared code repository.}, as given below.

\subsubsection{Europe dataset} We used the daily averaged power consumption in MW of 15 European countries\footnote{AT, BE, CH, DE, DK, ES, FI, FR, GB, IE, IT, NL, NO, PT, SE.} between 2015-2020 to build this dataset, which resulted in 2099 data points with 15 features. In other words, each data point corresponds to one day in the given five-year period, and each feature represents the aggregated daily consumption of the corresponding country. Fig. \ref{fig:samples} (blue points) visualizes this dataset using three countries' data, i.e. three data features. 
\subsubsection{Denmark dataset} We used the hourly averaged load and generation (solar, onshore wind and offshore wind) numbers (in MW) from the two bidding zones in Denmark in 2019 to build this dataset, which resulted in 8784 data points with 8 features. Therefore, each data point corresponds to an hour of a given day in 2019.

Please note that both datasets are treated as collections of \textit{i.i.d.} snapshots in time, not as time series. This treatment is relevant for various energy systems applications such as scenario testing for (cross-border) energy market studies and statistical modelling of load levels interconnection capacity planning.

\subsection{Two-Step Model Comparison Strategy}
Here, we introduce our model comparison strategy to assess the performance of the distribution models.
\subsubsection{Sample Comparison Tests}
In order to test the hypothesis of whether two sets of samples are coming from the same distribution or not, two multi-dimensional two-sample statistical tests were used: maximum mean discrepancy (MMD)\cite{gretton2012kernel} and energy\cite{szekely2013energy} tests\footnote{\url{https://github.com/josipd/torch-two-sample}}. These aim to test the hypothesis if the \textit{model samples} are coming from the same distribution as the \textit{test samples}. The test samples consist of the data points that we held out during the training of the models.

\subsubsection{Model Comparison Tests}
Two sample tests are designed to give smaller scores when the null hypothesis (samples are coming from the same distribution) is more likely. However, the test scores themselves are not easy to interpret numerically. In other words, the score alone cannot say if the best model amongst the candidate models is a good model or not. To overcome this, first, we obtained baseline scores by performing two-sample tests that compared the training data with the test data.

However, one drawback of this approach is that we \textit{randomly} split the training and test samples before the optimization. Thus, the obtained baseline score is merely a sample of a complex random process and can be misleading. Since conducting all the optimization procedures for different data splits is infeasible, we propose to use the random subsets of the train, test and generated sample sets to capture this effect as done in \cite{wang2022generating}. This Monte Carlo approach results in \textit{samples of sample comparison scores}, that are used to compare models.

Samples of test scores from the models can be compared to baseline score samples by using different statistics to calculate the difference between these sample distributions. For this, we use three two-sample tests as \textit{model comparison tests}: Kolmogorov-Smirnov (KS) \cite{massey1951kolmogorov}, Cramér–von Mises (CvM)\footnote{scipy.stats (v1.8.1) is used for KS and CvM tests.} \cite{anderson1962distribution}, and simple mean difference ($\Delta$Mean) tests. These result in \textit{model comparison scores} with smaller values indicating better performance.

Algorithm \ref{alg:two-step} contains a pseudo-algorithm of the two-step model comparison procedure. Here, $\mathcal{M}$, $\mathcal{X}^\text{train}$, $\mathcal{X}^\text{test}$, $N^\text{MC}$ and $r$ represents the set of the compared models, training set, test set, the number of Monte Carlo runs and the subsampling ratio, respectively. The operator $\stackrel{\text{n}}{\sim}$ means sampling n data points without replacement. As a result, we have a collection of model comparison scores $\mathcal{S}_{\text{T}^\text{s},\text{T}^\text{m}}^\text{M}$ for all model, sample comparison test and model comparison test triples, e.g. $\mathcal{S}_{\text{MMD},\text{CvM}}^{\pi-\text{KDE}}$.

\begin{algorithm}[t]
\caption{Two-Step Model Comparison Procedure}\label{alg:two-step}
\begin{algorithmic}[1]
\Require $\mathcal{M}$, $\mathcal{X}^\text{train}$, $\mathcal{X}^\text{test}$, $N^\text{MC},N^\text{model}\in\mathbb{N}$, $r\in(0,1)$
\State $\mathcal{T}^\text{S} = \{\text{MMD}, \text{Energy}\}$ \Comment{Sample comparison tests}
\State $n \gets r\times|\mathcal{X}^\text{test}|$ \Comment{\# subsamples}

\ForAll {$\text{T}^\text{s} \in \mathcal{T}^\text{S}$}
    \State $\mathcal{S}_{\text{T}^\text{s}}^\text{base} \leftarrow \{\}$
    \ForAll {$i \in 1,2,\dots,N^\text{MC}$}
        \State $\mathcal{D}^\text{test}\stackrel{\text{n}}{\sim}\mathcal{X}^\text{test}$; $\mathcal{D}^\text{base}\stackrel{\text{n}}{\sim}\mathcal{X}^\text{train}$ \Comment{Subsampling}
        \State $\mathcal{S}_{\text{T}^\text{s}}^\text{base} \leftarrow \mathcal{S}_{\text{T}^\text{s}}^\text{base} \cup \text{T}^\text{s}(\mathcal{D}^\text{test},\mathcal{D}^\text{base}) $ 
        \State \Comment{Collecting the baseline sample scores for test $\text{T}^\text{s}$ }
    \EndFor
\EndFor
\State
\ForAll {$\text{M} \in \mathcal{M}$}
    \State $\mathcal{X}^\text{model} \stackrel{N^\text{model}}{\sim} \text{M}$ \Comment{Taking samples from the model}
    \ForAll {$\text{T}^\text{s} \in \mathcal{T}^\text{S}$}
        \State $\mathcal{S}_{\text{T}^\text{s}}^\text{M} \leftarrow \{\}$
        \ForAll {$i \in 1,2,\dots,N^\text{MC}$}
            \State $\mathcal{D}^\text{test}\stackrel{\text{n}}{\sim}\mathcal{X}^\text{test}$; $\mathcal{D}^\text{model}\stackrel{\text{n}}{\sim}\mathcal{X}^\text{model}$ \Comment{Subsampling}
            \State $\mathcal{S}_{\text{T}^\text{s}}^\text{M} \leftarrow \mathcal{S}_{\text{T}^\text{s}}^\text{M} \cup \text{T}^\text{s}(\mathcal{D}^\text{test},\mathcal{D}^\text{model}) $
            \State \Comment{Collecting the sample scores of M for $\text{T}^\text{s}$}
        \EndFor
    \EndFor
\EndFor
\State
\State $\mathcal{T}^\text{M} = \{\text{KS}, \text{CvM}, \Delta\text{Mean}\}$ \Comment{Model comparison tests}
\ForAll {$(\text{T}^\text{s},\text{T}^\text{m},\text{M})  \in \mathcal{T}^\text{S} \times \mathcal{T}^\text{M} \times \mathcal{M}$}
    \State $\mathcal{S}_{\text{T}^\text{s},\text{T}^\text{m}}^\text{M} \leftarrow  \text{T}^\text{m}(\mathcal{S}_{\text{T}^\text{s}}^\text{base}, \mathcal{S}_{\text{T}^\text{s}}^\text{M}) $
    \State \Comment{Assigning model scores for every $\text{T}^\text{s}$ and M}
\EndFor

\end{algorithmic}
\end{algorithm}

\subsection{Experiment Settings}
\subsubsection{Dataset settings} Both datasets were randomly split into train and test sets with a ratio of 4:1 and normalized using z-score normalization.
\subsubsection{Optimizer settings} Adam was used for the gradient-based optimization. The convergence thresholds were set to $10^{-4}$ for all of the optimization algorithms. 
\subsubsection{Initialization settings} The initial bandwidths were assigned as $0.1$ both for A-KDE and $\pi$-KDE models \footnote{A logarithmic search was conducted for the initial bandwidths, and no value performs the best on all scores. Since the resulting scores are relatively close to each other, 0.1 was chosen as the most representative initial bandwidth value.} and initial weights were assigned as $\frac{1}{N}$ for the $\pi$-KDE model.
\subsubsection{Benchmark model settings} Full covariance GMMs were used as the model of comparison. Their numbers of components were set in a way that the total numbers of parameters were as close as possible to the numbers of parameters of A-KDE and $\pi$-KDE. The GMM models are denoted as GMM$_\text{A}$ and GMM$_\pi$, respectively. Thus, the set of models becomes $\mathcal{M}=\{\text{A-KDE},\pi\text{-KDE},\text{GMM}_\text{A},\text{GMM}_\pi\}$.
\subsubsection{Test settings} The number of samples was set to the size of the training set for all models, i.e. $N^\text{model}=|\mathcal{X}^\text{train}|$. We chose a subsampling rate of $r=0.5$ and set the number of Monte Carlo runs ($N^{\text{MC}}$) to 2000 and 1000 for Europe and Denmark datasets, respectively.
\section{Results and Discussion}

\subsection{Comparison of Training Speed}
First, we compared the speed differences between two optimization algorithms (modified EM and Adam) by measuring the elapsed times for convergence. Since the convergence speed of Adam also depends on the selection of the batch size and learning rate, we created a test grid that comprised combinations of several values of these hyperparameters\footnote{Batch size: [128, 256, 512, 1024]. Learning rate: [0.01, 0.05, 0.10]}. 

The results are illustrated in Fig. \ref{fig:speed_test}. As can be seen, the modified EM algorithm is not necessarily faster than Adam. However, the fact that the abundance of hyperparameter combinations that result in slower convergences makes the modified EM algorithm more favourable thanks to the absence of hyperparameters. Moreover, please note that the results are given in the logarithmic scale. Thus, the convergence time of Adam is significantly longer than the modified EM for most combinations, while the difference is negligible when Adam is faster.
\begin{figure}
    \includegraphics[width=1\linewidth]{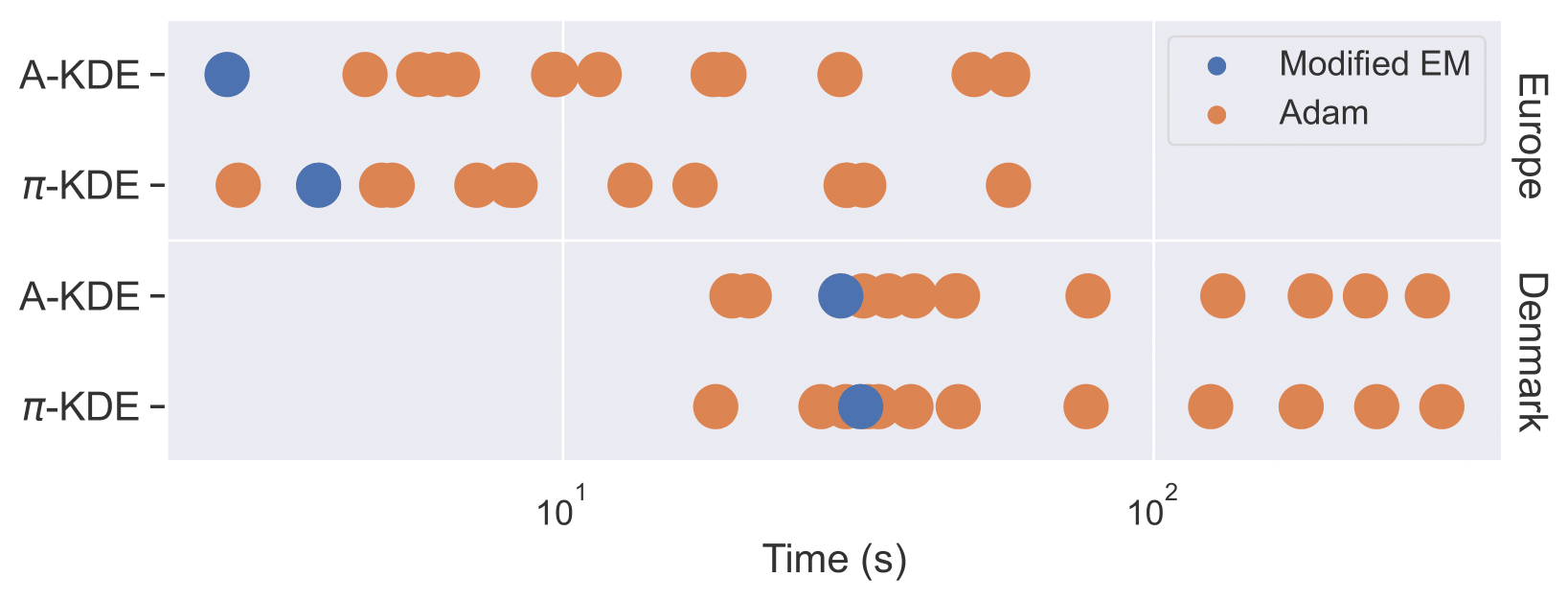}
    \vspace{-0.5cm}
    \caption{Convergence times of the modified EM and Adam.}
    \label{fig:speed_test}
\end{figure}

\subsection{Estimation Performance Comparison}

We trained all candidate models for each dataset. The empirical cumulative distribution functions (ECDFs) of the resulting sample test scores ($\mathcal{S}_{\text{T}^\text{s}}^\text{M}$ and $\mathcal{S}_{\text{T}^\text{s}}^\text{base}$) after the Monte Carlo runs are depicted in Fig. \ref{fig:scores}. The corresponding model comparison scores ($\mathcal{S}_{\text{T}^\text{s},\text{T}^\text{m}}^\text{M}$) are tabulated in Table \ref{tab:exp_result}. The final column indicates if the model is subject to singularity prevention.

In Fig. \ref{fig:scores}, we can see that the sample test scores of the models ($\mathcal{S}^\text{M}_{\text{T}^\text{s}}$) have similar variances with respect to their corresponding baseline sample test scores ($\mathcal{S}^\text{base}_{\text{T}^\text{s}}$). This eases the visual inspection by looking at the ordering of the ECDFs. Since we hypothesise that a lower score means better performance and that train and test data are drawn from the same distribution, having the baseline ECDFs on the leftmost for all orderings confirms our intuition.

First, we see from Table \ref{tab:exp_result} that A-KDE is consistently inferior to the $\pi$-KDE model. This supports the motivation of introducing kernel weights described in Section \ref{subs:loo_pi-kde}.

The similarity of distributions in Fig. \ref{fig:scores} suggests that all tested models are able to adequately represent the Denmark dataset. We interpret this as a result of the lower dimensionality and larger dataset size with respect to the Europe dataset. Nonetheless, $\pi$-KDE is more desirable in practice thanks to its singularity prevention guarantees. GMMs lack this prevention, and it is likely to encounter singularities, which we experienced occasionally during our experimentation.

For the Europe dataset, the best scores are obtained by the GMM$_\pi$ model. However, the $\pi$-KDE model also shows acceptable performance for this dataset. Qualitatively, we can see from Fig. \ref{fig:samples} that the complex nature of the marginals and `2-way marginals' of the selected countries are captured by the generated data. This makes the $\pi$-KDE model more favourable in cases where singularity prevention is crucial, like in edge-computing, in which the computation power is limited and re-running a failed GMM optimization might be costly in terms of time and energy. Similarly, (near) real-time applications might also require this singularity prevention due to the limited time budget to re-run a failed optimization attempt.

\begin{figure}[]
\centering
\subfloat[]{\label{subfig:scores_europe}\includegraphics[width=\linewidth]{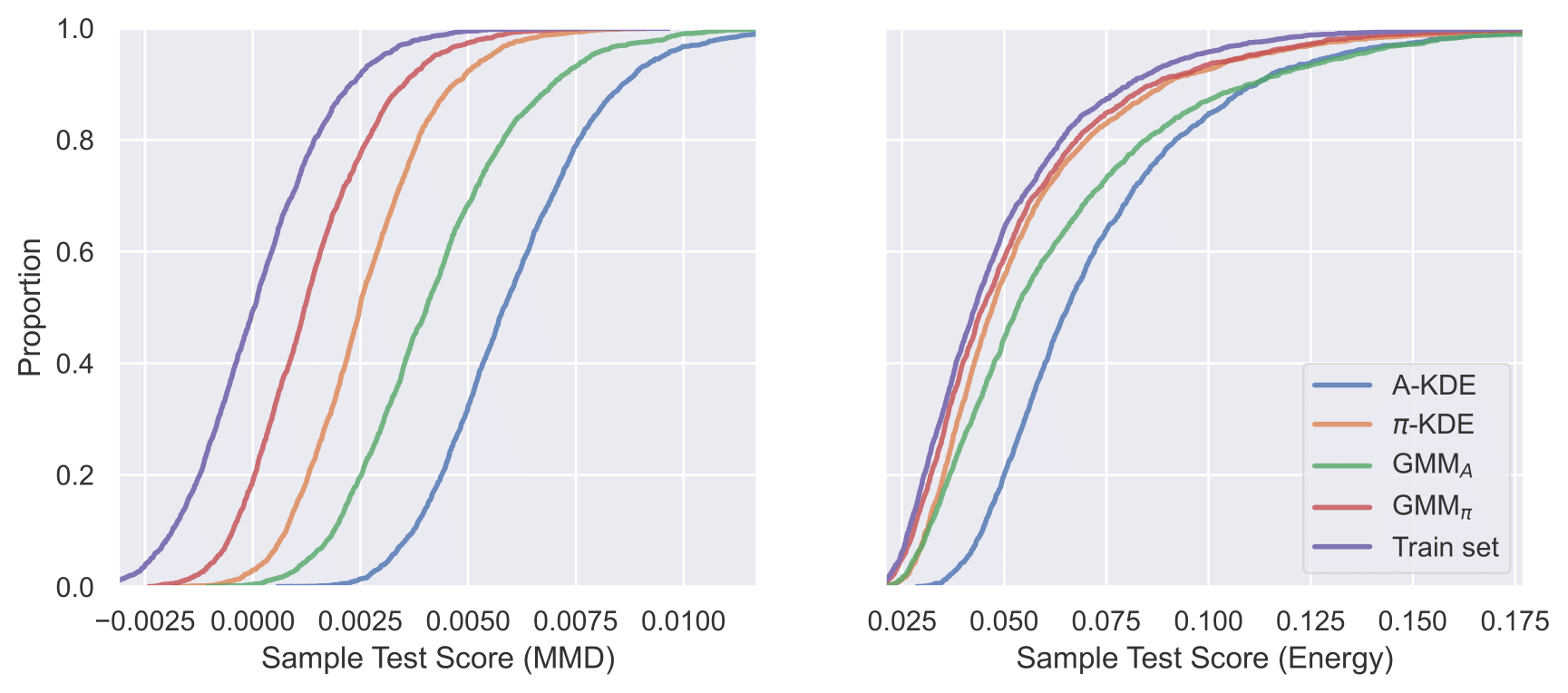}} \\
\subfloat[]{\label{subfig:scores_denmark}\includegraphics[width=\linewidth]{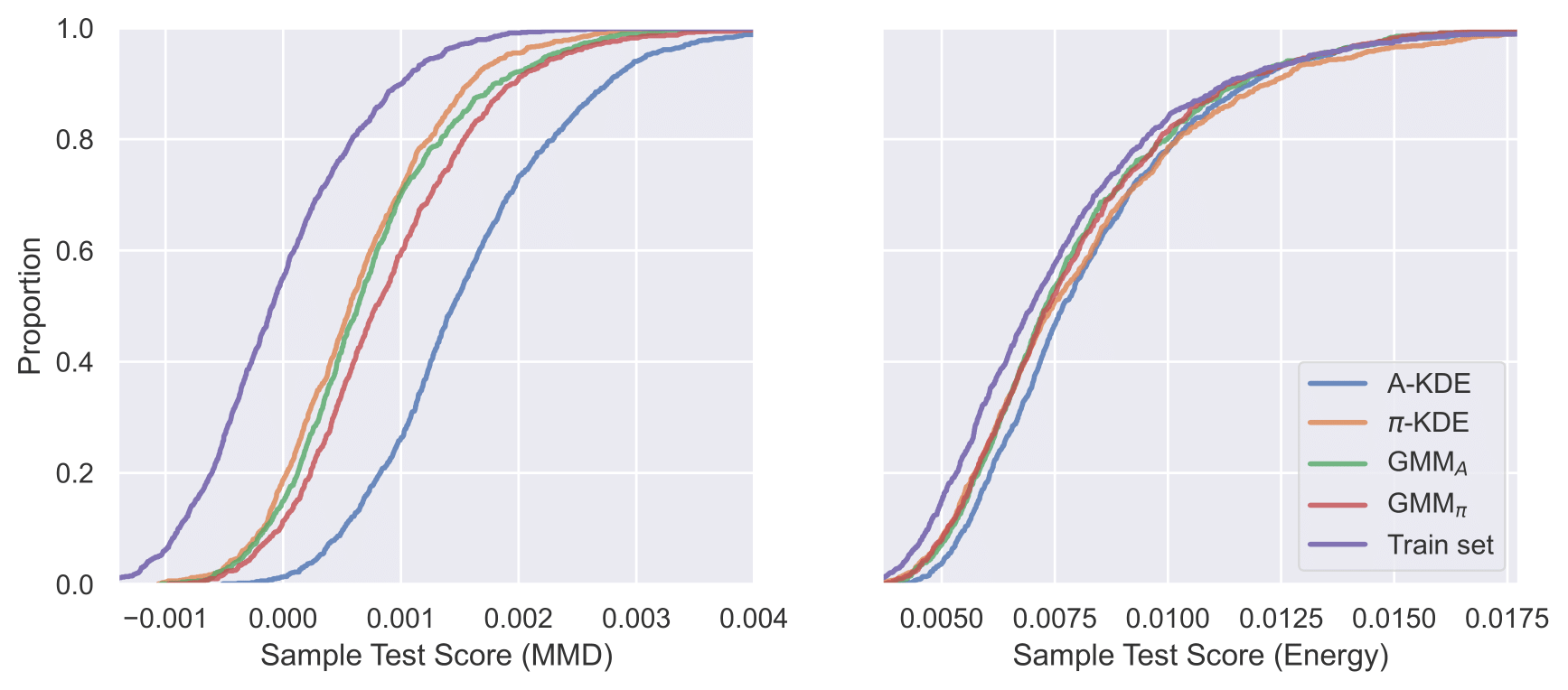}}%
\caption{ECDFs of sample test scores for (a) Europe and (b) Denmark datasets.}
\label{fig:scores}
\end{figure}

\begin{table}[t!]
\centering
\caption{Model Comparison Scores ($\mathcal{S}_{\text{T}^\text{s},\text{T}^\text{m}}^\text{M}$)}
\label{tab:exp_result}
\resizebox{\columnwidth}{!}{%
\begin{tabular}{@{}ccccccc@{}}
\toprule
\multirow{2}{*}{Dataset} & \multirow{2}{*}{\begin{tabular}[c]{@{}c@{}}Sample \\ Comparison \\ Test (T$^\text{s}$)\end{tabular}} & \multirow{2}{*}{Model}              & \multicolumn{3}{c}{Model Comparison Test (T$^\text{m}$)}     & \multicolumn{1}{l}{\multirow{2}{*}{\begin{tabular}[c]{@{}l@{}}Singularity\\ Prevention\end{tabular}}} \\ \cmidrule(lr){4-6}
                         &                                                                                                      &                                     & \begin{tabular}[c]{@{}c@{}}KS\\ $\times$10$^\text{-1}$\end{tabular}  & \begin{tabular}[c]{@{}c@{}}CvM \\ ~\end{tabular}           & \begin{tabular}[c]{@{}c@{}}$\Delta$Mean\\ $\times$10$^\text{-3}$\end{tabular} & \multicolumn{1}{l}{}                                                                                  \\ \midrule
\multirow{8}{*}{Europe}  & \multirow{4}{*}{MMD}                                                                                 & \multicolumn{1}{c|}{A-KDE}          & 9.16            & 317.88         & 5.91                      & \textbf{Yes}                                                                                          \\
                         &                                                                                                      & \multicolumn{1}{c|}{$\pi$-KDE}      & 5.82            & 167.79         & 2.50                      & \textbf{Yes}                                                                                          \\
                         &                                                                                                      & \multicolumn{1}{c|}{GMM$_\text{A}$} & 7.63            & 256.11         & 4.10                      & No                                                                                                    \\
                         &                                                                                                      & \multicolumn{1}{c|}{GMM$_\pi$}      & \textbf{3.08}   & \textbf{52.45} & \textbf{1.25}             & No                                                                                                    \\ \cmidrule(l){2-7} 
                         & \multirow{4}{*}{Energy}                                                                              & \multicolumn{1}{c|}{A-KDE}          & 4.46            & 95.96          & 23.99                     & \textbf{Yes}                                                                                          \\
                         &                                                                                                      & \multicolumn{1}{c|}{$\pi$-KDE}      & 1.24            & 6.95           & 5.97                      & \textbf{Yes}                                                                                          \\
                         &                                                                                                      & \multicolumn{1}{c|}{GMM$_\text{A}$} & 2.02            & 21.96          & 14.30                     & No                                                                                                    \\
                         &                                                                                                      & \multicolumn{1}{c|}{GMM$_\pi$}      & \textbf{0.56}   & \textbf{1.27}  & \textbf{3.34}             & No                                                                                                    \\ \midrule
\multirow{8}{*}{Denmark} & \multirow{4}{*}{MMD}                                                                                 & \multicolumn{1}{c|}{A-KDE}          & 6.84            & 111.03         & 1.58                      & \textbf{Yes}                                                                                          \\
                         &                                                                                                      & \multicolumn{1}{c|}{$\pi$-KDE}      & \textbf{3.79}   & \textbf{36.23} & \textbf{0.65}             & \textbf{Yes}                                                                                          \\
                         &                                                                                                      & \multicolumn{1}{c|}{GMM$_\text{A}$} & 4.19            & 42.52          & 0.75                      & No                                                                                                    \\
                         &                                                                                                      & \multicolumn{1}{c|}{GMM$_\pi$}      & 4.74            & 55.41          & 0.91                      & No                                                                                                    \\ \cmidrule(l){2-7} 
                         & \multirow{4}{*}{Energy}                                                                              & \multicolumn{1}{c|}{A-KDE}          & 1.57            & 5.51           & 0.68                      & \textbf{Yes}                                                                                          \\
                         &                                                                                                      & \multicolumn{1}{c|}{$\pi$-KDE}      & 0.95            & 2.24           & 0.57                      & \textbf{Yes}                                                                                          \\
                         &                                                                                                      & \multicolumn{1}{c|}{GMM$_\text{A}$} & 1.06            & 1.77           & \textbf{0.34}             & No                                                                                                    \\
                         &                                                                                                      & \multicolumn{1}{c|}{GMM$_\pi$}      & \textbf{0.93}   & \textbf{1.67}  & \textbf{0.34}             & No                                                                                                    \\ \bottomrule
                                                                                                                                                                                
\end{tabular}%
}
\end{table}

\section{Conclusion and Future Work}

Probabilistic modelling of power systems data is crucial for the future of systems operation and planning. This work introduced the data-copying phenomenon, which burdens such modelling by causing singular solutions. KDE-based models are employed to investigate this effect in a mathematically rigorous way. LOO-MLL criterion is proposed as a solution, and singularity prevention is guaranteed for two KDE-based models. Moreover, a modified EM optimization procedure is proposed for reliable training of the models. The models, along with benchmark GMM models, are tested on two power systems datasets using a novel testing procedure. The results show that the proposed models have an adequate modelling performance besides having singularity prevention guarantees.

Even though the KDE-based models are convenient for mathematical analysis, their applicability is limited since the number of kernels can easily be overwhelming for large datasets. A pruning mechanism might be the solution for this. Also, the isotropic nature of kernels can result in noisy samples if the data lies in a lower dimensional manifold. We plan to extend this work to kernels with full-covariance matrices. Lastly, as mentioned before, singularity caused by data-copying is a common problem in more advanced models too \cite{meehan2020non}, such as GMMs and variational autoencoders, so it is appealing to use the LOO-MLL in these models in future.

\bibliographystyle{IEEEtran}
\bibliography{refs.bib} 

\appendices

\section{Proof of Data-copying} \label{app:data_copy_proof}
\begin{proof} ($\Rightarrow$) Let $\sigma_{j'}\rightarrow 0^+$ while the rest of the bandwidths, $\{\sigma_j \}_{j\neq j'}$ have non-zero values. Plugging these into the scaled total log-likelihood without taking the limit results in
\begin{equation}
\begin{split}
& \sum_i \log \sum_{j} \mathcal{N}(\mathbf{x}_i;\mathbf{x}_j,\sigma^2_j\mathbf{I}) \stackrel{\text{c}}{=} \sum_i \log \left[ \frac{1}{\sigma_{j'}^d} \exp\left(\frac{-\Delta^2_{ij'}}{2\sigma^2_{j'}} \right)  + c_i\right] \\ &=
\log \left[ \frac{1}{\sigma_{j'}^d}  + c_{j'}\right] + \sum_{i\neq j'} \log \left[ \left(\frac{1}{\sigma_{j'}^2}\right)^{\frac{d}{2}} \exp\left(\frac{-\Delta^2_{ij'}}{2\sigma^2_{j'}} \right)  + c_i\right]
\end{split}
\end{equation}
where $c_i\coloneqq\sum_{j\neq j'}\mathcal{N}(\mathbf{x}_i;\mathbf{x}_j,\sigma^2_j\mathbf{I})$, $\Delta_{ij}\coloneqq\lVert \mathbf{x}_i - \mathbf{x}_j \rVert$ and $\stackrel{\text{c}}{=}$ means equal up to a constant. Taking the limit $\sigma_{j'}\rightarrow 0^+$, we get
\begin{equation}
    \begin{split}
        &\underset{\sigma_{j'}\rightarrow 0^+}{\lim} \sum_i \log \sum_{j} \mathcal{N}(\mathbf{x}_i;\mathbf{x}_j,\sigma^2_j\mathbf{I}) = \\ &
        -d  \underset{\sigma_{j'}\rightarrow 0^+}{\lim} \log(\sigma_{j'}) + \sum_{i\neq j'} \log(0+c_i) \rightarrow \infty
    \end{split}
\end{equation}
assuming there are no replica points in the dataset, i.e. $\Delta_{ij'}\neq 0, \forall i$. Otherwise, the second term also goes to infinity. In both cases, the total limit diverges to infinity, concluding the proof of the only if part.

($\Leftarrow$) Let us assume the opposite of the conclusion, i.e. $\exists\varepsilon>0: \sigma_j>\varepsilon, \forall j$. This makes the likelihoods of the datapoints under every kernel finite, i.e. $\exists c\in\mathbb{R}^+: \mathcal{N}(\mathbf{x}_i|\mathbf{x}_j,\sigma^2_j\mathbf{I}) \leq c, \forall i,j$. This also results in an upper-bounded total log-likelihood
\begin{equation}
    \frac{1}{N} \sum_i \log \sum_{j} \mathcal{N}(\mathbf{x}_i;\mathbf{x}_j,\sigma^2_j\mathbf{I}) \leq \log Nc
\end{equation}
which contradicts the initial statement and concludes the proof.
\end{proof}

\section{Proof of Data-copying Prevention by LOO-MLL}\label{app:loo_wellbehave}

\begin{proof}
    The assumption that there are no repeating data points implies $\min_{\{i,j:i\neq j\}}(\lVert \mathbf{x}_i - \mathbf{x}_j \rVert) > 0$. The remainder of the proposition can be formulated as 
    \begin{equation}
    \begin{split}
        \{\sigma_j^*\}_j =& ~\underset{\{\sigma_j\}_j}{\text{argmax}} \sum_i \log \sum_{j\neq i} \mathcal{N}(\mathbf{x}_i;\mathbf{x}_j,\sigma^2_j\mathbf{I}) \\
        &  \Longrightarrow \exists\varepsilon>0: \sigma_j^*\geq\varepsilon, \forall j
    \end{split}
    \end{equation}
    and express the negation as
    \begin{equation}\label{eq:proofB_contr}
    \begin{split}
        &\forall\varepsilon>0: \exists j: \sigma_j^*<\varepsilon ~ \wedge \\
        &  \left.g_j\right|_{\{\sigma^*_k\}_k} \vcentcolon= \left. \frac{\partial \sum_i \log \sum_{k\neq i} \exp(f_{ik})}{\partial \sigma_j}\right|_{\{\sigma^*_k\}_k}\leq0,
    \end{split}
    \end{equation}
    where $f_{ik}=-d\log \sigma_k - \frac{\Delta^2_{ik}}{2\sigma_k^2}$ and $\Delta_{ik}=\lVert \mathbf{x}_i - \mathbf{x}_k \rVert$. If we prove that \eqref{eq:proofB_contr} results in a contradiction, it concludes the proof. Note that the inequality in the (local) optimality condition covers the candidate solutions on the boundaries.

    The gradient expression $g_j$ can be derived as
    \begin{equation}\label{eq:g_j}
        g_j = \sum_{i\neq j} \left(\frac{\Delta^2_{ij}}{\sigma^3_j} - \frac{d}{\sigma_j} \right) w_{ij}
    \end{equation}
    where $w_{ij}=\frac{\exp(f_{ij})}{\sum_{k\neq i}\exp(f_{ik})}\in(0,1)$. Let us assume that $\sigma^*_{j'}<\varepsilon$ and express $g^*_{j'}=\left.g_{j'}\right|_{\{\sigma^*_k\}_k}\leq0$ using \eqref{eq:g_j} as
    \begin{equation}
        \begin{split}
        &\frac{1}{\varepsilon^2}<\frac{1}{\sigma^{*^2}_{j'}}\leq \frac{d \sum_{i\neq j'} w^*_{ij'}}{\sum_{i\neq j'}\Delta^2_{ij'} w^*_{ij'}} = \frac{d}{\sum_{i\neq j'} \frac{\Delta^2_{ij'}}{1+\sum_{l\neq i,j'}w_{lj'}}} \\
        & < \frac{d(N-1)}{\sum_{i\neq j'}\Delta^2_{ij'}} < \frac{d(N-1)}{N\min(\{\Delta^2_{ij'}\}_{i\neq j'})}.
        \end{split}
    \end{equation}
    Thus, the optimality statement in \eqref{eq:proofB_contr} takes the form of  
    \begin{equation}
        \forall\varepsilon>0: \exists j: \frac{N}{d(N-1)} \min(\{\Delta^2_{ij}\}_{i\neq j}) < \varepsilon^2.
    \end{equation}
    This is a contradiction as long as the non-repeating data point assumption ($\min_{\{i,j:i\neq j\}}(\lVert \mathbf{x}_i - \mathbf{x}_j \rVert) > 0$) holds.
    
\end{proof}

\end{document}